\documentclass{article}

\usepackage[preprint]{neurips_2024}
    \PassOptionsToPackage{numbers, compress}{natbib}

\usepackage[utf8]{inputenc} 
\usepackage[T1]{fontenc}    
\usepackage{hyperref}       
\usepackage{url}            
\usepackage{booktabs}       
\usepackage{amsfonts}       
\usepackage{nicefrac}       
\usepackage{microtype}      
\usepackage{xcolor}         
\usepackage{makecell}

\usepackage{graphicx}
\usepackage{amsmath}
\usepackage{mathtools}
\usepackage{multirow}
\usepackage{amsthm}
\usepackage{colortbl}
\usepackage{wrapfig}
\newtheorem{theorem}{Theorem}
\DeclareMathOperator*{\argmin}{arg\,min}
\DeclareUnicodeCharacter{221A}{\ensuremath{\sqrt{}}}
\DeclareMathOperator{\diag}{diag}

\usepackage{listings}
\usepackage{xcolor}

\definecolor{codegray}{gray}{0.95}
\definecolor{pykeyword}{rgb}{0.0,0.0,0.6}   
\definecolor{pystring}{rgb}{0.2,0.5,0.2}    
\definecolor{pycomment}{rgb}{0.5,0.5,0.5}   
\definecolor{pyfunc}{rgb}{0.55,0.0,0.55}    

\title{Differentiable Convex Optimization Layers in Neural Architectures: Foundations and Perspectives}
\makeatletter
\renewcommand{\@noticestring}{}
\makeatother
\author{%
  Calder K. Katyal\\
  Undergraduate\\
  Yale University\\
  New Haven, CT 06511 \\
  \texttt{calder.katyal@yale.edu} \\
}
\begin{document}

\maketitle

\begin{abstract}
    The integration of optimization problems within neural network architectures represents a fundamental shift from traditional approaches to handling constraints in deep learning. While it is long known that neural networks can incorporate soft constraints with techniques such as regularization, strict adherence to hard constraints is generally more difficult. A recent advance in this field, however, has addressed this problem by enabling the direct embedding of optimization layers as differentiable components within deep networks. This paper surveys the evolution and current state of this approach, from early implementations limited to quadratic programming, to more recent frameworks supporting general convex optimization problems. We provide a comprehensive review of the background, theoretical foundations, and emerging applications of this technology. Our analysis includes detailed mathematical proofs and an examination of various use cases that demonstrate the potential of this hybrid approach. This work synthesizes developments at the intersection of optimization theory and deep learning, offering insights into both current capabilities and future research directions in this rapidly evolving field. 
\end{abstract}
\section{Background}
\label{sec:intro}
Imagine we want to train a model to solve a structured task given a set of input and output sequences, without explicit knowledge of the constraints or rules involved. A typical approach might be to use a feedforward network, such as a Convolutional Neural Network (CNN), to learn patterns from the data. To attempt to learn and conform with the rules, we might try techniques such as regularization or a robust loss function that penalizes rulebreaking. Nevertheless, without a perfectly tuned reward signal from the loss function, the model will tolerate invalid solutions, leading to overfitting as the model may latch onto irrelevant patterns that don’t generalize. If we could incorporate a learnable component into the model that intuitively learns the rules of the task and ensures the output conforms with them, we would expect improved generalization, as the constraints would be independent of specific input distributions and become globally optimized over the training process. However, there is no general mechanism for incorporating complex, learnable hard constraints in a vanilla feedforward network. Only simple constraints—such as the simplex constraint in the softmax layer—are easily enforced. To address this limitation, \citet{optnet} proposes a foundational method of embedding these constrained optimization problems into the neural architecture as distinct layers. In this layer, the network can output the solution to the constrained optimization problems as part of its forward pass, where the constraints are parametrized by the previous layers; furthermore, it can propagate gradient information through these constraints in the backward pass to better understand the nature of the task. This allows the architecture to enforce either explicitly defined or implicitly learned hard constraints within an end-to-end differentiable framework. 

In this survey, we focus on the development of convex optimization layers as a means to enforce hard constraints within neural network architectures, a capability that has traditionally been challenging to achieve. Starting with the pioneering work of \citet{optnet}, which introduced the embedding of quadratic programming layers, we trace the progression of this idea into more general convex optimization frameworks as in \citet{differentiableconvexoptimizationlayers}. We examine the mathematical principles that enable these layers to remain differentiable, the computational methods required for their efficient implementation, and the integration of these methods into existing neural network workflows. By exploring both theoretical advancements and practical applications, this survey aims to provide a deeper understanding of how convex optimization layers enhance the capacity of neural networks to solve structured problems while adhering to complex constraints.

The methods employed in these convex optimization layers are extensions of long-standing theory in the subject of convex optimization, the full treatment of which can be found in \citet{boyd2004convex}. The key to these models is to formulate a learnable optimization problem using the networks parameters and output a solution which can be fed into subsequent layers. The first technique developed for this was an adaptation of \textit{argmin differentiation} in which implicit differentiation is used to obtain the gradients from the Karush-Kuhn-Tucker (KKT) matrix associated with the problem. Subsequent works such as \citet{differentiableconvexoptimizationlayers} transform the problem into a cone problem and use the recently developed technique in \citet{conedifferentiation} to differentiate through the solution map of the cone program. Of course, as QPs are a subset of cone programs, this method is immediately transferable to the optimization layers in \citet{optnet}.

\subsection{Convex Optimization}

A convex optimization layer simply takes the form of a parameterized convex problem, where the output of the layer is the solution \(x^\star\) obtained by minimizing the objective function subject to the constraints.

In particular, we define a (parametrized) convex problem as

\begin{equation}
\begin{aligned}
    \text{minimize} \quad & f_0(x; \theta) \\
    \text{subject to} \quad & f_i(x; \theta) \leq 0, \quad i = 1, \ldots, m, \\
    & a_i^T x = b_i, \quad i = 1, \ldots, p,
\end{aligned}
\end{equation}

where $f_0(x; \theta)$ and $f_i(x; \theta)$, $i = 1, \ldots, m$, are convex functions, and $a_i^T x - b_i$ are affine equality constraints \citet{boyd2004convex}. Here, $x \in \mathbb{R}^n$ is the optimization variable, $\theta \in \mathbb{R}^p$ is a parameter vector, and $a_i \in \mathbb{R}^n$, $b_i \in \mathbb{R}$. The feasible set of this problem is given by \[\mathcal{D} = \bigcap_{i=0}^m \text{dom}(f_i) \cap \bigcap_{i=1}^m \{x \mid f_i(x) \leq 0\} \cap \bigcap_{i=1}^p \{x \mid a_i^T x = b_i\}.
\] and is convex as it is the intersection of convex sets. A convex optimization layer aims to optimize the parameters \(\theta\) such that the solution \(x^\star\), which is the optimal value determined by solving the convex optimization problem parameterized by \(\theta\), minimizes a scalar loss function \(L(x^\star; \theta)\). This ensures that \(\theta\) is learned in a way that aligns the solution \(x^\star\) with the overall task objectives encoded in \(L\), enabling the optimization layer to contribute effectively to the network. Before such a general formulation was considered, the original set of optimization problems as in \citet{optnet} were restricted to convex quadratic programs (QPs) of the form

\begin{equation}
\begin{aligned}
    \text{minimize} \quad & \frac{1}{2} x^T P x + q^T x + r \\
    \text{subject to} \quad & Gx \preceq h, \\
                             & Ax = b,
\end{aligned}
\end{equation}

where \(x \in \mathbb{R}^n\) is the optimization variable, \(P \in \mathbb{S}_+^n\) is a positive semidefinite matrix, \(q \in \mathbb{R}^n\), \(r \in \mathbb{R}\), \(G \in \mathbb{R}^{m \times n}\), \(h \in \mathbb{R}^m\), \(A \in \mathbb{R}^{p \times n}\), and \(b \in \mathbb{R}\), forming a polyhedral feasible set. 

\subsection{Argmin Differentiation}

The method of implicit differentiation used in \citet{optnet} builds off a technique called \textbf{argmin differentiation} that has been previously applied to differentiate through particular optimization problems. The first major treatment of argmin differentiation in a general setting is found in \citet{gould2016differentiating}. While this approach provides theoretical groundwork for differentiating through optimization layers, it treats equality and inequality constraints separately, rather than handling them simultaneously. We briefly review their key results for three cases: unconstrained optimization problems, optimization problems with only equality constraints, and optimization problems with only inequality constraints. The proofs use results from calculus and are omitted here. 

Consider the following general optimization problem: 

\begin{equation}
g(x) = \argmin_{y} f(x,y)
\end{equation}

where $f: \mathbb{R} \times \mathbb{R}^n \rightarrow \mathbb{R}$ is continuous and twice-differentiable. The general results found in \citet{gould2016differentiating} are provided below:

\textbf{Unconstrained} 

The gradient with respect to $x$ can be computed through implicit differentiation:

\begin{equation}
g'(x) = -f_{YY}(x,g(x))^{-1}f_{XY}(x,g(x))
\end{equation}

where $f_{YY} = \nabla^2_{yy}f(x,y)$ and $f_{XY} = \frac{\partial}{\partial x}\nabla_y f(x,y)$.

\textbf{Equality-Constrained}

Suppose we equality constraints of the form $Ay = b$, where $A \in \mathbb{R}^{m \times n}$ and $b \in \mathbb{R}^m$. For the general case, let $y_0$ be any vector satisfying $Ay_0 = b$ and let $F$ be a matrix whose columns span the null-space of $A$. Then:

\begin{equation}
g'(x) = -F(F^T f_{YY}(x,g(x))F)^{-1}F^T f_{XY}(x,g(x))
\end{equation}

In particular, when $\text{rank}(A) = m$ ($A$ has full row rank), a direct approach yields:
\begin{equation}
g'(x) = (H^{-1}A^T(AH^{-1}A^T)^{-1}AH^{-1} - H^{-1})f_{XY}(x,g(x))
\end{equation}
where $H = f_{YY}(x,g(x))$. 

\textbf{Inequality-Constrained}

Suppose we have inequality constraints $f_i(x,y) \leq 0$, $i=1,\ldots,m$. Using the barrier method (see \citet{boyd2004convex} for a more comprehensive discussion), the constrained problem can be approximated as:

\begin{equation}
\text{minimize}_y \quad tf_0(x,y) - \sum_{i=1}^m \log(-f_i(x,y))
\end{equation}

where $t > 0$ controls the approximation accuracy. The gradient can then be computed as:

\begin{equation}
g'(x) \approx -(tf_{YY}(x,g(x)) + \phi_{YY}(x,g(x)))^{-1}(tf_{XY}(x,g(x)) + \phi_{XY}(x,g(x)))
\end{equation}

where $\phi(x,y) = \sum_{i=1}^m \log(-f_i(x,y))$ is the barrier function.

While providing a mathematical foundation for optimization layers, this approach lacks a unified treatment of mixed constraints. There are also other inefficiencies, such as the use of the barrier method as an approximation to the inequality-constrained case, which can be numerically unstable and computationally expensive. The inversion of the Hessian matrices in the equality-constrained case can furthermore be unsuited for large-scale problems.

Thus, \citet{optnet} uses an adapated, more general version of the technique and \citet{differentiableconvexoptimizationlayers} uses the even more efficient technique of cone differentiation, which is discussed below. This new approach is shown to be general in \citet{differentiableconvexoptimizationlayers}, as a large subclass of convex optimization problems known as disciplined convex programming (DCP) problems admit reformulations as a cone problems. 

\subsection{Differentiation Through Cone Programs}

While argmin differentiation provides a theoretical foundation for differentiating through optimization problems, a more computationally efficient approach leverages the structure of cone programs. This procedure was developed in \citet{conedifferentiation} and is summarized here. The primal form of a cone program takes the form:
\begin{equation}
\begin{aligned}
\text{minimize} \quad & c^T x \\
\text{subject to} \quad & Ax + s = b \\
                        & s \in K
\end{aligned}
\end{equation}

where $K \subseteq \mathbb{R}^m$ is a proper convex cone (a convex set closed under positive scaling), $x \in \mathbb{R}^n$ is the primal optimization variable, and $s \in \mathbb{R}^m$ is a slack variable that captures inequality constraints. The Lagrangian function for this cone program is given by:
\begin{equation}
\mathcal{L}(x,s,y) = c^T x + y^T(Ax + s - b)
\end{equation}
Therefore, the dual optimization problem is of the form:
\begin{equation}
\begin{aligned}
\text{maximize} \quad & -b^T y \\
\text{subject to} \quad & A^T y + c = 0 \\
                        & y \in K^*
\end{aligned}
\end{equation}

where $y \in \mathbb{R}^m$ represents the dual variables and $K^* = {y \in \mathbb{R}^m : y^T x \geq 0 \text{ for all } x \in K}$ is the dual cone of $K$ (in particular, this is the set of all vectors \( y \) that form non-negative inner products with every vector \( x \) in the cone \( K \)). 

The next step is to understand how optimal solutions depend on and vary with problem parameters. To do this systematically, we need to derive the Karush-Kuhn-Tucker conditions for the system. 

\subsubsection{KKT Conditions for Convex Optimization Problems}

The Karush-Kuhn-Tucker (KKT) conditions form a cornerstone of optimization theory, defining optimality criteria for constrained optimization problems. The full treatment can be found in \citet{boyd2004convex}; the main results are summarized below. In the context of convex optimization, these conditions are both necessary and sufficient for optimality. Consider a convex optimization problem where we aim to minimize a convex objective function $f_0(\mathbf{x})$ subject to inequality constraints $f_i(\mathbf{x}) \leq 0$ and equality constraints $h_i(\mathbf{x}) = 0$. The problem can be expressed as:

\begin{align*}
    \text{Minimize } & \quad f_0(\mathbf{x}) \\
    \text{subject to } & \quad f_i(\mathbf{x}) \leq 0, \quad i = 1, \dots, m, \\
    & \quad h_i(\mathbf{x}) = 0, \quad i = 1, \dots, p.
\end{align*}

Here, the \textbf{primal problem} refers to this original formulation, while the \textbf{dual problem} arises by introducing Lagrange multipliers $\lambda_i$ (for $f_i$) and $\nu_i$ (for $h_i$), allowing us to construct a lower bound on the objective function. The dual problem maximizes this bound and provides insights into the sensitivity of the optimal solution to changes in the constraints. The KKT conditions unify these perspectives by characterizing the relationships between the primal and dual problems at optimality. The following theorem is a standard result in convex optimization:

\begin{theorem}[KKT Conditions for Convex Problems]
\label{thm:kkt}
For the above convex optimization problem, where $f_0$ and $f_i$ are convex and $h_i$ are affine, a point $(\mathbf{x}^*, \lambda^*, \nu^*)$ is optimal if and only if the following conditions hold:
\begin{enumerate}
    \item \textbf{Primal feasibility:} $f_i(\mathbf{x}^*) \leq 0$ and $h_i(\mathbf{x}^*) = 0$ for all $i$.
    \item \textbf{Dual feasibility:} $\lambda_i^* \geq 0$ for all $i$.
    \item \textbf{Complementary slackness:} $\lambda_i^* f_i(\mathbf{x}^*) = 0$ for all $i$.
    \item \textbf{Stationarity:}
    \[
    \nabla f_0(\mathbf{x}^*) + \sum_{i=1}^m \lambda_i^* \nabla f_i(\mathbf{x}^*) + \sum_{i=1}^p \nu_i^* \nabla h_i(\mathbf{x}^*) = 0.
    \]
\end{enumerate}
\end{theorem}

Furthermore, as long as there is one strictly feasible point, the primal and dual solutions coincide at optimality via Slater's condition, as strong duality is obtained.

\subsubsection{Solution Map Decomposition}

Thus, in the case of a cone program, we have that:

\begin{enumerate}
    \item \textbf{Primal feasibility} is given by the primal constraints:
    \[
    A x + s = b, \quad s \in K,
    \]

    \item \textbf{Dual feasibility} is given by the dual constraints:
    \[
    A^T y + c = 0, \quad y \in K^*,
    \]
    
    \item \textbf{Complementary slackness} is given by:
    \[
    s^T y = 0,
    \]
    which ensures that the slack variable \( s \) and the dual variable \( y \) are orthogonal, meaning \( s \) and \( y \) cannot both be strictly positive at any point.

    \item The \textbf{stationary} condition is equivalent to the dual feasibility condition, as it is given by:
    \[
    \nabla_x \mathcal{L}(x, s, y) = c + A^T y = 0,
    \]
\end{enumerate}

Thus the solution map to the primal-dual conic program is of the form 
\[
S: \mathbb{R}^{m \times n} \times \mathbb{R}^m \times \mathbb{R}^n \to \mathbb{R}^n \times \mathbb{R}^m \times \mathbb{R}^m,
\]
\[
\text{where } S(A, b, c) = (x^*, y^*, s^*) \in 
\left\{
\begin{aligned}
    & (x, y, s) \mid A x + s = b, \; A^T y + c = 0, \\
    & s \in K, \; y \in K^*, \; s^T y = 0
\end{aligned}
\right\}
\]

In particular, we aim to find the derivative of the solution map \(S\) with respect to the problem data \((A, b, c)\), denoted DS(A, b, c). To do this, decompose the solution map \(S: (A,b,c) \mapsto (x,y,s)\) into three components: \(S = \phi \circ \psi \circ Q\), where

\[
Q : \mathbb{R}^{m \times n} \times \mathbb{R}^m \times \mathbb{R}^n \to \mathcal{Q}
\]
\[
s : \mathcal{Q} \to \mathbb{R}^N
\]
\[
\phi : \mathbb{R}^N \to \mathbb{R}^n \times \mathbb{R}^m \times \mathbb{R}^m
\]In this formulation, $Q$ maps the problem data to a skew-symmetric matrix (a matrix where \(Q = -Q^T\)) that encodes the problem structure:
\begin{equation}
Q(A,b,c) = \begin{pmatrix} 
0 & A^T & c \\
-A & 0 & b \\
-c^T & -b^T & 0
\end{pmatrix}.
\end{equation}

Furthermore, $s$ yields a solution to the so-called homogeneous self-dual embedding introduced in \citet{Yeselfdualemb} and \citet{Xuselfdualemb}, which a mathematical formulation  that transforms a primal-dual optimization problem into a single system of equations. In particular, solving this system ensures that the resulting solution satisfies all the primal-dual KKT conditions. Finally, $\phi$ maps a solution of the homogeneous self-dual embedding to the solution of the primal-dual pair. 

By the chain rule, we have 

\[
DS(A, b, c) = D\phi(z) \, Ds(Q) \, DQ(A, b, c).
\]

To compute the derivative of the solution map $DS(A,b,c)$, the authors utilize three key transformations ensure mathematical tractability. The first step maps the problem data perturbations $(dA,db,dc)$ into a skew-symmetric matrix $dQ$ by directly arranging them into the corresponding block structure, preserving the critical skew-symmetric property required for feasibility. The most challenging component is computing $Ds$ - this requires solving a large linear system $Mdz = -g$ that encodes how these perturbations affect the solution. Here $M = ((Q-I)D\Pi(z)+I)/w$ combines the problem matrices with projection derivatives, scaled by the solution parameter $w$. The vector $g = dQ\Pi(z/|w|)$ represents how perturbations affect the normalized residual map that measures optimality. Since directly forming and inverting $M$ would be computationally infeasible for large problems, the authors employ LSQR (Least Squares QR) - an iterative method introduced in \citet{lsqr}. LSQR gradually builds up the solution by combining QR matrix factorization with a technique called bidiagonalization, requiring only matrix-vector products rather than storing or inverting the full matrix $M$. The final transformation $D\phi$ converts the computed changes $dz$ into perturbations in the optimization variables $(dx,dy,ds)$ through carefully constructed projections onto the dual cone. For computing derivatives in the backward pass needed for deep learning, this process is reversed using the same efficient numerical techniques. The authors demonstrate this approach scales to problems with millions of variables through an efficient implementation avoiding any explicit matrix operations.

\section{OptNet: Integrating QP Layers Into Neural Architecture}
We begin our discussion of convex optimization layers by deriving the key results in \citet{optnet}, the first major attempt to integrate such layers into neural networks in an end-to-end fashion. The goal is to make the mathematical theory underpinning them rigorous and explore the intuition behind the various components. We will then discuss the expressive power of these layers, and conclude with a discussion of their limitations. 

\subsection{Method}

The general idea in \citet{optnet} cast a convex QP as an optimization layer, where the output  $z_{i+1}$ of the layer is the solution of the problem and the constraints are parametrized by the output of the previous layer $z_i$. In particular, for optimization variable $z \in \mathbb{R}^n$, we have: 

\[
z_{i+1} = \text{argmin}_{z} \frac{1}{2} z^T Q(z_i) z + q(z_i)^T z
\]
\[
\text{subject to} \quad A(z_i) z = b(z_i), \quad G(z_i) z \leq h(z_i).
\]

where \( Q \in \mathbb{R}^{n \times n} \succeq 0 \) ensures convexity of the objective, and \( q \in \mathbb{R}^n \), 
\( A(z_i) \in \mathbb{R}^{m \times n} \), \( b(z_i) \in \mathbb{R}^m \), \( G(z_i) \in \mathbb{R}^{p \times n} \), and \( h(z_i) \in \mathbb{R}^p \) are problem data. 

The first step is to construct the Lagrangian. For notational clarity, we temporarily suppress the $z_i$ dependence for the rest of the analysis:

\[
L(z,\nu,\lambda) = \frac{1}{2}z^T Q z + q^T z + \nu^T(Az - b) + \lambda^T(Gz - h)
\]

where $\nu \in \mathbb{R}^m, \lambda \in \mathbb{R}^p$ are dual variables and $\lambda \geq 0$. By \ref{thm:kkt}, the optimal solution $(z^*, \nu^*, \lambda^*)$ must satisfy: 

\begin{enumerate}
    \item \textbf{Primal feasibility:} \\
    The primal variables $z^*$ must satisfy the equality and inequality constraints:
    \[
    Az^* = b, \quad Gz^* \leq h.
    \]

    \item \textbf{Dual feasibility:} \\
    As we already noted, the dual variables $\lambda^*$ must be non-negative, because they are associated with inequality constraints:
    \[
    \lambda^* \geq 0.
    \]

    \item \textbf{Complementary slackness:} \\
    For each inequality constraint, either the constraint is active ($Gz^* = h$) or the corresponding dual variable is zero ($\lambda^* = 0$):
    \[
    \lambda_i^* (Gz^* - h)_i = 0 \quad \forall i \in \{1, \ldots, p\}.
    \]

    \item \textbf{Stationarity:} \\
    The gradient of the Lagrangian with respect to the primal variables $z$ must vanish:
    \[
    \nabla_z L(z^*, \nu^*, \lambda^*) = Qz^* + q + A^T\nu^* + G^T\lambda^* = 0.
    \]
\end{enumerate}

Vectorizing the conditions and noting that the stationary condition encompasses the primal feasibility requirement $Gz^* \leq h$, the necessary and sufficient conditions for optimality (given $\lambda \geq 0)$ are: 

\begin{equation}
    \begin{aligned}
        Qz^* + q + A^T\nu^* + G^T\lambda^* = 0 \\
        Az^* - b = 0\\
        \diag({\lambda_i^*})(Gz^* - h)_i = 0
    \end{aligned}
\end{equation}

To approximate how the solution $(z^*, \nu^*, \lambda^*)$ changes with respect to perturbations in the problem parameters, we take the differential:

\begin{equation}
    \begin{aligned}
        dQ z^* + Q dz + dq + dA^T \nu^* + A^T d\nu + dG^T \lambda^* + G^T d\lambda &= 0, \\
        dA z^* + A dz - db &= 0, \\
        \diag(\lambda^*) d(Gz^* - h) + \diag(Gz^* - h) d\lambda &= 0,
    \end{aligned}
\end{equation}

Which admits the matrix formulation:

\begin{equation}
    \begin{bmatrix}
        Q & G^T & A^T \\
        \diag(\lambda^*) G & \diag(Gz^* - h) & 0 \\
        A & 0 & 0
    \end{bmatrix}
    \begin{bmatrix}
        dz \\
        d\lambda \\
        d\nu
    \end{bmatrix}
    =
    -\begin{bmatrix}
        dQ z^* + dq + dG^T \lambda^* + dA^T \nu^* \\
        \diag(\lambda^*) dG z^* - \diag(\lambda^*) dh \\
        dA z^* - db
    \end{bmatrix}
\end{equation}

Note that the goal is not to explicitly calculate the Jacobian matrices with respect to the parameters (e.g.$\frac{\partial z^*}{\partial b} \in \mathbb{R}^{n \times m}$), but rather to compute the Jacobian-vector products used in backpropogation (e.g. $\frac{\partial \ell}{\partial z^*} \frac{\partial z^*}{\partial b}$). This avoids explicitly constructing large Jacobian matrices.
To efficiently compute the Jacobian-vector product required in the backward pass, the differential conditions can be leveraged. 

The directional derivatives needed in the chain rule can be found by noting that only $\frac{\partial \ell}{\partial z^*}$ is relevant for propogating gradients and is agnostic to perturbations in the constraints. Therefore, we can write: 

\begin{equation}
\begin{bmatrix}
d_z \\
d_\lambda \\
d_\nu
\end{bmatrix}
=
-
\begin{bmatrix}
Q & G^\top \diag(\lambda^*) & A^\top \\
G & \diag(Gz^* - h) & 0 \\
A & 0 & 0
\end{bmatrix}^{-1}
\begin{bmatrix}
\frac{\partial \ell}{\partial z^*} \\
0 \\
0
\end{bmatrix}.
\end{equation}

The gradient updates are given by: 
\begin{equation}
    \begin{aligned}
        \nabla_Q \ell &= \frac{1}{2}\left(d_z z^{*\top} + z^* d_z^\top\right), \\
        \nabla_A \ell &= d_\nu z^{*\top} + \nu^* d_z^\top, \\
        \nabla_G \ell &= \diag(\lambda^*)d_\lambda z^{*\top} + \lambda^* d_z^\top, \\
        \nabla_q \ell &= d_z, \\
        \nabla_b \ell &= -d_\nu, \\
        \nabla_h \ell &= -\diag(\lambda^*) d_\lambda.
    \end{aligned}
\end{equation}

The authors implement a batched QP solver to make the approach practical for deep learning. Standard solvers like Gurobi and CPLEX, while optimal for solving individual QPs, process QPs sequentially on CPUs. This creates a massive bottleneck when applied to mini-batches in neural networks. To address this, the authors developed a GPU-bound solver that solves all QPs in a mini-batch in parallel, dramatically improving efficiency.

The computational bottleneck lies in factorizing the KKT matrix, which encodes the relationships between the primal and dual variables as well as the constraints. This factorization is necessary because it allows efficient solving of the linear systems required by the primal-dual interior point method. The complexity of factorization is cubic (\( \mathcal{O}(n^3) \)) in the number of variables, but once computed during the forward pass, it is reused to solve linear systems during backpropagation, which has quadratic complexity (\( \mathcal{O}(n^2) \)).

The GPU-based solver exploits parallelism by solving all QPs in the mini-batch simultaneously. For a batch size of \( N \), the parallel implementation reduces the overall complexity to \( \mathcal{O}(n^3 + N n^2) \), compared to the sequential cost of \( \mathcal{O}(N n^3) \) with CPU-based solvers. By sharing the expensive KKT factorization across the batch, the approach ensures both forward and backward passes are computationally efficient and scalable to large mini-batches, making it feasible for end-to-end training in deep learning systems.

 The full details are found in \citet{optnet}, and the solver is based on a primal-dual interior point method developed in \citet{mattingley2012cvxgen}.

\subsection{Expressivity}

A key benefit of these layers is their expressive power. The key results obtained in the paper are summarized and proved here with greater rigor:

The first result is that an OptNet layer can represent any  elementwise piecewise linear function, in particular $\text{ReLU}(\cdot)$.

\begin{theorem}
    An elementwise piecewise linear function $f: \mathbb{R}^n \to \mathbb{R}^{n}$ with $k$ linear regions can be represented by OptNet layer using $\mathcal{O}(nk)$ parameters. In particular, we can represent the ReLU layer $z_{i+1} = \max\{Wz_i + b, 0\}$ (where $W \in \mathbb{R}^{n \times m}, b \in \mathbb{R}^n$) by an OptNet layer with $\mathcal{O}(mn)$ parameters. 
\end{theorem}

\begin{proof}
    A piecewise linear function can be expressed in terms of maximum operators. For instance, the univariate piecewise linear function $f(x)$ with $k$ linear regions can be written as:
    \[
    f(x) = \sum_{i=1}^k w_i \max\{a_i x + b_i, 0\},
    \]
    where $w_i \in \{-1, 1\}$ and $a_i, b_i \in \mathbb{R}$. This representation is derived by iteratively constructing linear segments based on breakpoints.

    To encode this function in an OptNet layer, we define the following quadratic program:
    \[
    \min_{z, t \in \mathbb{R}^k} \|t\|_2^2 + \|z - w^T t\|_2^2 \quad \text{subject to } a_i x + b_i \leq t_i, \, \forall i \in \{1, \ldots, k\}.
    \]
    The optimization ensures that each $t_i$ takes the value $\max\{a_i x + b_i, 0\}$ due to the constraint structure and objective minimization, while $z = w^T t$ reproduces the output $f(x)$. By applying this elementwise, we can extend the formulation to $n$ dimensions, maintaining a parameter complexity of $\mathcal{O}(nk)$.

    For the specific case of ReLU, where $z = \max\{Wx + b, 0\}$, we observe that this is equivalent to the optimization:
    \[
    \min_{z} \|z - (Wx + b)\|_2^2 \quad \text{subject to } z \geq 0.
    \]
    Here, the parameter complexity is dominated by $W \in \mathbb{R}^{n \times m}$, resulting in $\mathcal{O}(mn)$ parameters.
\end{proof}

The authors moreover suggests that an OptNet layer is actually more expressive than a two-layer $ReLU$ network

\begin{theorem}
    There exists a function $f(z): \mathbb{R}^n \to \mathbb{R}$, represented by an OptNet layer with $p$ parameters, that the two-layer ReLU network $g(x) = \sum_{i=1}^k w_i \max\{a_i^{\top} z + b_i, 0\}$ cannot globally represent over $\mathbb{R}$ and which require $\mathcal{O}(c^p)$ parameters to approximate over any finite region.
\end{theorem}

\begin{proof}
Consider $f(x) = \max\{a_1^\top x, a_2^\top x, a_3^\top x\}$ with linearly independent $a_1,a_2,a_3 \in \mathbb{R}^2$. This function admits an OptNet formulation:
\begin{equation*}
\begin{aligned}
\min_{z} \quad & z^2 \\
\text{subject to} \quad & a_i^\top x \leq z, \quad i = 1, 2, 3
\end{aligned}
\end{equation*}

The non-differentiable set $\mathcal{N}_f$ of $f$ forms three bounded line segments, terminating at their intersections. However, each ReLU unit contributes a line of non-differentiability to $\mathcal{N}_{f'}$ that extends to infinity. The geometric incompatibility of bounded and unbounded lines makes exact representation impossible.

For a fixed radius $r$, achieving $\epsilon$-approximation requires covering the terminating boundaries with ReLU units. Due to the unbounded nature of ReLU boundaries, the network must use an exponentially growing number of units as $\epsilon$ decreases to maintain accuracy throughout the region. Since the OptNet representation uses only $p$ parameters, this establishes the $\mathcal{O}(c^p)$ bound.
\end{proof}

\section{Differentiable Convex Optimization Layers}

The next major advancement in the field was the extension to a more general class of convex optimization layers in \citet{differentiableconvexoptimizationlayers}. The method furthermore automates the conversion of the optimization problem to canonical form and uses the efficient method of differentiating cone problems found in \citet{conedifferentiation}. We will discuss the methods used in \citet{differentiableconvexoptimizationlayers} which make this possible, which include the invention of a new grammar for convex problems and way of canonicalizing the problem to the corresponding (differentiable) cone problem. 

\subsection{Formal Grammar}
To ensure the method is easily generalizable to a wide class of convex optimization problem, it is necessary to define a rigorous grammar for constructing convex optimization problems. 

\subsubsection{Disciplined Convex Programming}
The grammar developed in \citet{differentiabilitysolutionconvexoptimization} is an extension of a well-known grammar, \textbf{disciplined convex programming (DCP)} \citet{dcp}, which encompasses a large class of convex opimization problems. DCP relies on atomic functions - basic components with explicitly defined mathematical properties - and a composition framework that preserves convexity through principled combination of these atoms.

Each atomic function must specify three key attributes: its curvature classification (whether it behaves as an affine, convex, or concave function), its monotonicity characteristics per input argument, and its valid domain of operation. These properties enable systematic verification of convexity when atoms are combined.

The theoretical foundation lies in a composition principle: Let $h$ represent a convex function mapping from $\mathbb{R}^k$ to $\mathbb{R}$, with monotonicity properties defined by two key index sets - $I_1$ for arguments where $h$ is nondecreasing, and $I_2$ for those where it is nonincreasing. When composed with functions $g_i$ mapping from $\mathbb{R}^n$ to $\mathbb{R}$ that satisfy specific curvature conditions (convex for indices in $I_1$, concave for $I_2$, and affine elsewhere), the resulting function $f(x) = h(g_1(x),...,g_k(x))$ maintains convexity.

This composition theorem enables verification through rule-checking rather than complex mathematical analysis. While this means some valid convex problems cannot be expressed directly in the framework, the atom library concept provides extensibility - new functions can be added with verified properties, expanding the scope of expressible problems while maintaining the verification guarantees. In particular, DCP is the grammar behind many popular libraries such as Python's \texttt{cvxpy}.

\subsubsection{Disciplined Parametrized Programming}

A key observation in \citet{differentiableconvexoptimizationlayers} is that DCP does not ensure that the mapping from problem parameters to solutions is structured in a way that enables efficient analytical differentiation through the optimization process. Thus, it is necessary to modify DCP slightly to create a new grammar which they call disciplined paramatrized programming (DPP). The main feature of this grammar is that the produced program can be reduced (canonicalized) to a special form called \textbf{affine-solver-affine} (ASA) form. Backpropogation through the optimization layer is agnostic to the operations in the canonicalization of a DPP to ASA form and thus allows automatic differentiation throughout the network. 

simplifying differentiation by enabling efficient computation of gradients without directly backpropagating through the complex steps of problem transformation,

DPP, like DCP, relies on a set of atomic functions with explicitly defined curvature (constant, affine, convex, or concave) and per-argument monotonicities. These atoms are combined according to composition rules derived from the composition theorem for convex functions, ensuring that the resulting expressions preserve convexity. Problems in both DCP and DPP can be represented as computational trees, where the nodes represent atomic functions and the leaves correspond to variables, constants, or parameters. This compositional structure enables systematic verification of convexity by analyzing the curvature of individual components and the rules governing their combination.

DPP differs from DCP in how it treats parameters and imposes additional restrictions on their usage to enable more expressive parameterized problems while maintaining convexity guarantees. 

A disciplined parametrized program is an optimization problem of the form:
\[
\begin{aligned}
\text{minimize} \quad & f_0(x, \theta) \\
\text{subject to} \quad & f_i(x, \theta) \leq \tilde{f}_i(x, \theta), \quad i = 1, \ldots, m_1, \\
& g_i(x, \theta) = \tilde{g}_i(x, \theta), \quad i = 1, \ldots, m_2,
\end{aligned}
\]
where \( x \in \mathbb{R}^n \) represents the optimization variables, \( \theta \in \mathbb{R}^p \) are parameters, \( f_i \) are convex functions, \( \tilde{f}_i \) are concave functions, and \( g_i, \tilde{g}_i \) are affine functions. Parameters are symbolic constants with known properties, such as sign or monotonicity, but without fixed numerical values. An expression is said to be \textit{parameter-affine} if its leaves do not include variables and is affine in its parameters; an expression is said to be \textit{parameter-free} or \textit{variable-free} if it does not have parameters or variables respectively.

Namely, DPP differs from DCP in two ways:

\begin{enumerate}
    \item \textbf{Treatment of Parameters}.** In DCP, parameters are treated as constants with fixed curvature and monotonicity properties. This restricts their use in constructing optimization problems. DPP, on the other hand, treats parameters as affine objects, similar to variables, allowing them to participate more flexibly in parameterized expressions. This change enables DPP to express problems involving linear relationships between parameters and variables while preserving the convexity guarantees.
    \item  \textbf{Generalized Composition Rules} DPP relaxes some restrictions on atomic operations involving parameters. For example, in DCP, the product atom \(\phi_{\text{prod}}(x, y) = xy\) is classified as affine only if one of its arguments is a constant (i.e., variable-free). Under DPP, this rule is expanded, also allowing the product to be classified as affine if one argument is parameter-affine and the other is parameter-free. This expansion enables the representation of a wider range of parameterized optimization problems.
\end{enumerate}

Every DPP program satisfies the rules of DCP, but the converse is not true. DPP extends the expressive power of DCP by introducing additional flexibility in the treatment of parameters and broadening the conditions under which operations are considered convex, concave, or affine. Despite these differences, both frameworks rely on the same composition-based approach to systematically verify the structure of optimization problems.

\subsection{Canonicalization to Affine-Solver-Affine Form}

Consider a disciplined convex program with variable $x \in \mathbb{R}^n$ parametrized by $\theta \in \mathbb{R}^p$. The solution map $S: \mathbb{R}^p \to \mathbb{R}^n$ maps parameters to solutions. To enable backpropagation through this layer in a neural network, we must determine how to compute the adjoint of the derivative $D^T S(\theta)$. Here $Df(x)$ denotes the derivative (Jacobian) of a function $f$ evaluated at point $x$, while $D^T f(x)$ denotes its adjoint - the transpose of the Jacobian matrix. The adjoint is critical as it enables efficient computation of gradients during backpropagation.

The solution map $S$ can be expressed as a composition of three maps: $S = R \circ s \circ C$, where canonicalizer $C$ maps parameters to cone program data $(A, b, c)$, cone solver $s$ produces solution $\tilde{x}^\star$, and retriever $R$ maps $\tilde{x}^\star$ to original solution $x^\star$. When both $C$ and $R$ are affine maps, we say the problem is in \textbf{affine-solver-affine (ASA) form}. For this composition, we can apply the chain rule to derivatives. For functions $f$ and $g$, recall that $D(g \circ f)(x) = Dg(f(x))Df(x)$. Taking the adjoint of both sides gives $(D(g \circ f)(x))^T = (Df(x))^T(Dg(f(x)))^T$. Applying this twice to our three-function composition yields:

\[
D^T S(\theta) = D^T C(\theta)D^T s(A, b, c)D^T R(\tilde{x}^\star)
\]

In standard DCP canonicalization, computing these derivatives requires tracking dependencies through a complex tree of transformations. DCP recursively expands each nonlinear atom into its graph implementation. For instance, $\|x\|_2$ becomes a new variable $t$ with constraint $(t,x) \in \mathcal{Q}^{n+1}$ where $\mathcal{Q}^{n+1}$ is the second-order cone. While this preserves convexity, it creates intricate dependencies between parameters and problem data. Even a simple term like $\theta x$ becomes problematic - DCP treats $\theta$ as a constant, obscuring how the solution changes with $\theta$ and requiring backpropagation through the full canonicalization procedure.

The grammar of DPP introduced in \citet{differentiableconvexoptimizationlayers} is a subset of DCP that ensures the canonicalization map $C$ and retriever $R$ are affine. This enables representing these maps as sparse matrices that only need to be computed once. DPP achieves this by treating parameters as affine objects rather than constants and restricting parameter operations to preserve affine structure. Under DPP, products require one term to be parameter-free, ensuring problem data depends affinely on parameters.

The following theorem formalizes this sparse matrix representation:

\begin{theorem}[Canonicalizer Representation]
The canonicalizer map $C$ for a disciplined parametrized program can be represented by a sparse matrix $Q \in \mathbb{R}^{n \times p+1}$ and sparse tensor $R \in \mathbb{R}^{m \times n+1 \times p+1}$, where $m$ is the constraint dimension. Let $\tilde{\theta} \in \mathbb{R}^{p+1}$ denote the concatenation of $\theta$ and 1. Then the cone program data is given by:
\[
c = Q\tilde{\theta}, \quad [A \quad b] = \sum_{i=1}^{p+1} R_{[:,:,i]}\tilde{\theta}_i
\]
\end{theorem}

\begin{proof}
We construct $Q$ and $R$ via reduction on the affine expression trees representing the canonicalized problem. Consider a root node $f$ with arguments $g_1,\ldots,g_n$. We obtain tensors $T_1,\ldots,T_n$ representing $f$'s linear action on each argument, then recursively process each subtree $g_i$ to obtain tensors $S_1,\ldots,S_n$. The DPP rules ensure that for each $i$, either $(T_i)_{j,k,l}=0$ for $l\neq p+1$ or $(S_i)_{j,k,l}=0$ for $l\neq p+1$.

Define operation $\psi(T_i,S_i)$ as:
\[
\psi(T_i,S_i) = \begin{cases}
   \sum_{l=1}^{p+1} (T_i)_{[:,:,p+1]}(S_i)_{[:,:,l]} & \text{ if } (T_i)_{j,k,l}=0 \text{ for } l\neq p+1 \\
   \sum_{l=1}^{p+1} (T_i)_{[:,:,l]}(S_i)_{[:,:,p+1]} & \text{ otherwise}
\end{cases}
\]

The tree rooted at $f$ evaluates to $S_0 = \sum_{i=1}^n \psi(T_i,S_i)$. The recursion bottoms out at leaf nodes:
\begin{itemize}
    \item Variable leaf $x \in \mathbb{R}^d$ yields $T \in \mathbb{R}^{d\times n+1 \times 1}$ with $T_{i,j,1} = 1$ if $i$ maps to $j$ in the variable vector
    \item Parameter leaf $p \in \mathbb{R}^d$ yields $T \in \mathbb{R}^{d\times 1 \times p+1}$ with $T_{i,1,j} = 1$ if $i$ maps to $j$ in the parameter vector  
    \item Constant leaf $c \in \mathbb{R}^d$ yields $T \in \mathbb{R}^{d\times 1 \times 1}$ with $T_{i,1,1} = c_i$
\end{itemize}
This construction yields the desired sparse representations $Q$ and $R$.
\end{proof}

This representation fundamentally changes how we handle derivatives. Rather than backpropagating through graph implementations and constraint reformulations, we compute derivatives using sparse matrix operations. Given parameters $\theta$, we obtain cone program data via $c = Q\tilde{\theta}$ and $[A \quad b] = \sum_{i=1}^{p+1} R_{[:,:,i]}\tilde{\theta}_i$, apply the cone solver for $\tilde{x}^\star$, and recover $x^\star$ through the affine retriever. For backpropagation, $D^T C(\theta)$ and $D^T R(\tilde{x}^\star)$ are simply transposes of our sparse matrices. The only substantial computation is differentiating through the cone solver $s$, which is accomplished using the method in \citet{conedifferentiation} (namely, by implicitly differentiating the KKT optimality conditions). This approach provides analytical derivatives while avoiding repeated canonicalization.

Note that the QP layer given in \citet{optnet} is a specific case of the convex optimization layers in \citet{differentiableconvexoptimizationlayers} and can be represented in DPP.

\section{Implementation}

The differentiable optimization layers in \citet{differentiableconvexoptimizationlayers} are easily implemented using the accompanying python library \texttt{cvxpylayers}, which support both PyTorch and TensorFlow. In particular, the authors show that the OptNet QP Layer 

\[
\begin{aligned}
\text{minimize} \quad & \frac{1}{2} z^\top Q x + q^\top z \\
\text{subject to} \quad & Az = b, \\
                        & Gz \leq h,
\end{aligned}
\]

where \( z \in \mathbb{R}^n \) is the optimization variable, and the problem data are \( Q \in \mathbb{R}^{n \times n} \succeq 0\) , \( q \in \mathbb{R}^n \), \( A \in \mathbb{R}^{m \times n} \), \( b \in \mathbb{R}^m \), \( G \in \mathbb{R}^{p \times n} \), and \( h \in \mathbb{R}^p \). We can implement this with:

\begin{lstlisting}[language=Python]
Q_sqrt = cp.Parameter((n, n))
q = cp.Parameter(n)
A = cp.Parameter((m, n))
b = cp.Parameter(m)
G = cp.Parameter((p, n))
h = cp.Parameter(p)
z = cp.Variable(n)
obj = cp.Minimize(0.5*cp.sum_squares(Q_sqrt*z) + q.T @ z)
cons = [A @ z == b, G @ z <= h]
prob = cp.Problem(obj, cons)
layer = CvxpyLayer(prob, parameters=[Q_sqrt, q, A, b, G, h], 
                   variables=[z])
\end{lstlisting}

The efficiency is comparable to OptNet's original implementation but the implementation is far simpler. The original paper \citet{differentiableconvexoptimizationlayers} further demonstrates how simple layers (e.g. ReLU, Softmax) can be easily framed as optimization problems and implemented, and a full set of examples can be found at the GitHub. repository.\footnote{https://github.com/cvxgrp/cvxpylayers}

\section{Applications}

We examine several key domains where convex optimization layers provide novel capabilities beyond traditional neural architectures. We focus particularly on applications that exploit the layer's ability to enforce hard constraints while maintaining end-to-end differentiability.

\subsection{Structured Prediction}

A fundamental application of optimization layers is enforcing complex structural dependencies in prediction tasks. The seminal work of \citet{optnet} demonstrated this through the task of learning to solve Sudoku puzzles from input-output pairs alone. The architecture learns to represent the game's logical constraints through the optimization layer's constraint matrices $A$ and $G$, enabling strict enforcement of rules during inference without explicit encoding.

More formally, consider a structured prediction task where outputs $y$ must satisfy some (potentially unknown) constraint set $\mathcal{C}$. Traditional approaches often rely on relaxations or penalties to approximately enforce constraints. In contrast, an optimization layer can represent this directly as:

\begin{equation}
\begin{aligned}
\text{minimize}_y \quad & f(y; x, \theta) + \frac{1}{2}y^TQy \\
\text{subject to} \quad & A(x,\theta)y = b(x,\theta) \\
& G(x,\theta)y \leq h(x,\theta)
\end{aligned}
\end{equation}

where $f$ captures the prediction objective and the constraints parametrized by $\theta$ implicitly learn $\mathcal{C}$ through training data. This formulation ensures predictions strictly satisfy the learned constraints while remaining differentiable.

\subsection{Signal Processing}

One natural application that \citet{optnet} touches on is signal processing, where optimization is already commonly used. Consider the task of signal denoising - removing noise from a corrupted signal to recover the original clean version. Traditional approaches like total variation (TV) denoising formulate this as a convex optimization problem:

\begin{equation}
\text{minimize}_z \quad \frac{1}{2}\|y - z\|_2^2 + \lambda\|Dz\|_1
\end{equation}

where $y$ is the noisy signal, $z$ is the denoised output, $D$ is a differencing operator, and $\lambda$ controls the strength of the smoothing. While effective, this requires manual tuning of $\lambda$ and uses a fixed differencing operator $D$.

By embedding this optimization within a neural network, we can instead learn both the regularization strength $\lambda$ and the structure of the operator $D$ from data. Given pairs of clean and noisy signals, the network can be trained end-to-end to minimize reconstruction error. This allows the model to automatically adapt to the statistics of the noise and signal.

\subsection{Linear Model Sensitivity to Adversarial Attacks}

\citet{differentiableconvexoptimizationlayers} further show that the differentiable structure of optimization layers provides a powerful framework for analyzing how machine learning models behave under adversarial modifications to training data. The general idea is to model an adversary as someone who can deliberately perturb the training data to degrade the model's performance on a separate test set. This can be formulated as a bilevel optimization problem, where the outer problem identifies the most damaging perturbations, and the inner problem represents the process of retraining the model on the altered data. Mathematically, this is expressed as:

\begin{equation}
\begin{aligned}
\text{minimize}_{\delta} \quad & L_{\text{test}}(\theta^*(\mathbf{x} + \delta)) \\
\text{subject to} \quad & \|\delta\|_\infty \leq \epsilon, \\
\text{where} \quad & \theta^*(\mathbf{x}) = \argmin_\theta \frac{1}{N}\sum_{i=1}^N \ell(\theta; x_i, y_i) + r(\theta).
\end{aligned}
\end{equation}

Here, the inner optimization reflects the usual model training process, where the parameters \(\theta\) are learned to minimize a convex objective, such as a regularized loss function, on the (potentially perturbed) training data. The outer optimization represents the adversary’s goal of finding the perturbation \(\delta\), constrained by an \(\ell_\infty\)-norm bound \(\epsilon\), that maximizes the test loss \(L_{\text{test}}\). This framework can simulate real-world adversarial scenarios like data poisoning attacks.

To better understand this, think of the model as being retrained on slightly "corrupted" versions of the training data. The adversary's goal is to subtly shift the training points—within a small allowable range—so that the model performs poorly on a separate test set. For instance, a logistic regression model might learn a decision boundary based on its training data. By tweaking certain training points in the direction of the gradient of the test loss, the adversary can "nudge" this boundary, leading to poorer classification accuracy on the test data.

The adversarial perturbations are computed using the gradients of the test loss with respect to the training points. For each point, the adversary applies an update of the form:
\[
x_i \leftarrow x_i + \epsilon \, \text{sign}(\nabla_{x_i} L_{\text{test}}(\theta^*)),
\]
where \(\nabla_{x_i} L_{\text{test}}(\theta^*)\) indicates the direction in which the test loss increases the most with respect to \(x_i\). The adversary moves each training point slightly in this direction to maximize the negative impact on the model's performance.

This bilevel structure is differentiable because we can compute gradients through both the inner training process and the outer optimization. Differentiable optimization layers make this possible by treating the solution to the inner problem (training) as a function of the input data, enabling efficient gradient computation. With this setup, one can precisely study how small data modifications propagate through the training process to affect the model's generalization performance, providing insights into both the vulnerabilities of the model and potential strategies for defense.

\subsection{Future Directions}

Several promising theoretical and applied directions remain to be explored:

\textbf{Robust Machine Learning and AI:}
Differentiable convex optimization layers can enhance the robustness of AI models by integrating adversarial defense mechanisms and uncertainty modeling directly into training. This enables better generalization to unseen data, resilience against data poisoning, and adaptability to real-world conditions where data is imperfect or dynamic.

\textbf{Autonomous Systems and Control Optimization:}
These layers are well-suited for optimizing control policies in robotics, self-driving vehicles, and other autonomous systems. They can solve trajectory planning and real-time decision-making problems under physical, safety, and operational constraints, ensuring efficiency and reliability in dynamic environments.

\textbf{Dynamic Resource Management:}
From energy systems like smart grids to telecommunications and logistics, differentiable convex layers can optimize the allocation of limited resources in real time. Their ability to handle constraints and dynamically adapt to changing conditions makes them ideal for applications requiring efficient resource scheduling or balancing.

\textbf{Complex Decision-Making in Applied Domains:}
Fields such as finance, healthcare, and physics-informed modeling can leverage these layers for solving constrained optimization problems within larger machine learning systems. Applications include portfolio management, personalized treatment plans, and systems obeying physical laws, where decision-making involves balancing competing objectives under strict constraints.

The key advantage in these domains is the layer's ability to maintain feasibility with respect to domain constraints while enabling gradient-based learning of problem structure. This offers a powerful framework for combining domain knowledge with data-driven approaches in a mathematically rigorous way.

\section{Limitations and Future Work}

We discuss the limitations of each framework and suggest possibilities for future work. 

\subsection{Limitations of OptNet}

While the OptNet framework exhibits several strengths, its practicality is limited by several factors. The first limitation is the cubic computational complexity associated with solving optimization problems exactly. This arises from the need to factorize the Karush-Kuhn-Tucker (KKT) system, which dominates the cost of solving the QPs. In practice, this makes the OptNet layer unsuitable for problems with high-dimensional decision variables or constraints. Specifically, the authors of \citet{optnet} note that the framework is generally practical for problems with fewer than 1000 hidden variables in the optimization layer, with substantially smaller limits required for applications demanding real-time computation. This stands in contrast to the quadratic complexity of standard feedforward neural network layers, such as fully connected or convolutional layers, highlighting a key bottleneck for scaling OptNet layers. Moreover, the original framework did not invoke structures in the data matrices such as sparsity, which could possibly increase efficiency. 

Another key problem is that tuning the parameters in an OptNet layer is highly nontrivial, primarily due to the invariances embedded in the parameter space. As an example, consider the equality constraint \(A x = b\), where \(A \in \mathbb{R}^{m \times n}\) and \(b \in \mathbb{R}^m\). If one row of \(A\), say the \(i\)-th row \(a_i^\top\), is scaled by a constant \(c > 0\), and simultaneously, the corresponding \(i\)-th entry in \(b\), \(b_i\), is scaled by the same constant \(c\), the constraint remains unchanged. Mathematically, the constraint 
\[
a_i^\top x = b_i
\]
transforms to 
\[
(c a_i^\top) x = (c b_i),
\]
which is equivalent because the scaling cancels out when checking feasibility. While this invariance does not affect the feasible set or optimal solution, it introduces challenges during training. Gradient-based optimization may attempt to update \(A\) and \(b\) along such invariant directions, leading to flat regions in the loss landscape where changes to \(A\) and \(b\) do not result in any change to the optimization layer’s output. This creates difficulty for gradient descent, as it struggles to identify informative updates, slowing convergence or requiring more aggressive tuning of hyperparameters such as learning rates. These challenges are compounded when multiple invariances interact, which can produce a degenerate parameter space that complicates training and necessitates careful initialization and optimization strategies to ensure stable learning. 

The final limitation is that OptNet is geared towards a particular class of optimization problems (QPs), and moreover requires users to explicitly transform their problems into the appropriate canonical form. This is addressed in \citet{differentiableconvexoptimizationlayers}.

\subsection{Limitations of Differentiable Convex Optimization Layers (Agrawal et al.)}

A fundamental limitation is that the approach hinges on solving a conic problem in each forward pass, and then differentiating through its solution map, incurring a computational cost that often scales at least cubicly in the problem dimension (e.g., factorizing KKT systems is typically $\mathcal{O}(n^3)$). This complexity can become impractical for large-scale instances, limiting the layer’s utility in real-time or high-throughput settings.

Moreover, the method only applies to disciplined parametrized programs, which excludes a range of nonconvex or even certain convex formulations not easily expressed within their grammar. This modeling restriction may force significant and sometimes unnatural problem reformulations. Another issue arises at points of non-uniqueness or degeneracy in the solution: the solution map is not guaranteed to be differentiable, so the backward pass may produce heuristics rather than exact gradients. Finally, while the authors’ abstraction streamlines integration into standard deep learning frameworks, it still depends on robust conic solvers, which can fail, struggle with numerical ill-conditioning, or return only approximate solutions

\subsection{Future Directions}

A key direction is to develop solver routines that exploit problem structure—such as sparsity patterns of the problem data to drive down the cubic complexity currently incurred when factorizing KKT systems. Such improvements would directly benefit both general ASA-form layers and specialized frameworks like OptNet’s QP layers. Another immediate need is to extend the DPP grammar to handle a richer class of convex problems and certain structured nonconvex ones, reducing the overhead of manual, fragile re-formulations. On the theoretical side, ensuring stable and unique solutions—via mild regularization or strategic perturbations—would guarantee that gradients are well-defined and numerically stable, even at degenerate points. Finally, tighter coupling with advanced conic and QP solvers, including better warm-starting strategies, could improve solver reliability and runtime performance, making these approaches scalable to larger, more complex tasks.

\section{Conclusion}

This survey has examined the theoretical foundations and practical implementations of differentiable optimization layers in neural architectures. Starting from quadratic programming layers, we traced the evolution to general convex optimization layers through key theoretical advances in differentiating through optimization problems. The development progressed from argmin differentiation to the more general cone program approach, enabling efficient handling of broader constraint classes.

Our analysis has revealed both significant progress and remaining challenges. While robust mathematical frameworks now exist for incorporating hard constraints in end-to-end trainable architectures, the cubic computational complexity of core operations remains a key bottleneck for scaling to high-dimensional problems. The parameter space exhibits problematic invariances that complicate training, and extensions beyond convex settings remain largely unexplored.

Looking forward, several promising directions emerge: more efficient solvers exploiting problem structure, extensions of the DPP grammar to broader problem classes, and deeper theoretical understanding of optimization-neural network interactions. As the field matures, we expect optimization layers to become an increasingly essential tool in deep learning, particularly for applications demanding hard constraints and structured predictions.




\newpage

\end{document}